\documentclass{article}

\usepackage[main, final]{neurips_2026}

\usepackage[utf8]{inputenc} 
\usepackage[T1]{fontenc}    
\usepackage{hyperref}       
\usepackage{url}            
\usepackage{booktabs}       
\usepackage{amsfonts}       
\usepackage{nicefrac}       
\usepackage{microtype}      
\usepackage{xcolor}         
\usepackage{graphicx}
\usepackage{wrapfig}
\usepackage{adjustbox}
\usepackage{multirow}
\usepackage{amsmath}
\usepackage{amsthm}
\usepackage{amsfonts}
\usepackage{amsmath}
\usepackage{algorithm}
\usepackage{algorithmic}
\usepackage{amsfonts}
\usepackage{amssymb}
\usepackage{bbm}
\newtheorem{theorem}{Theorem}
\newtheorem{lemma}{Lemma}

\newtheorem{proposition}{Proposition}
\newtheorem{corollary}{Corollary}

\newcommand{\rqref}[1]{\textbf{#1}}
\title{Rethinking Structured Generation: Can Graph-Based Reasoning Resolve Ambiguity?}

\author{
  Ratun Rahman \qquad Atit Pokharel\\
  Department of Electrical and Computer Engineering\\
  The University of Alabama in Huntsville\\
  Huntsville, AL 35816\\
  \texttt{\{rr0110, ap1284\}@uah.edu} 
}

\begin{document}

\maketitle

\begin{abstract}
Structured generation is typically formulated as a deterministic mapping from input to a single structured output. This assumption is fundamentally misaligned with natural language, which often admits multiple structurally distinct yet semantically valid interpretations. As a result, existing methods rely on single-reference supervision that collapses ambiguity into one prediction, leading to unstable and inconsistent outputs under underspecified inputs. We reformulate structured generation as an \textit{ambiguity-aware inference problem}, where the objective is to reason over a set of plausible structured outputs rather than predict a single solution. We show that standard training inherently suppresses valid alternatives, and instead adopt a multi-valid learning framework that preserves ambiguity during optimization. To operationalize this formulation, we construct a candidate space of structured outputs, represent each candidate as a graph capturing entities and their relationships, and perform selection via graph-based reasoning over semantic alignment and relational consistency. This enables principled comparison across competing interpretations beyond token-level likelihoods. Experiments on structured generation benchmarks demonstrate consistent improvements in accuracy, stability, and robustness under varying levels of ambiguity. These results highlight the importance of modeling ambiguity explicitly and suggest a shift from single-output prediction to set-level reasoning in structured generation.
\end{abstract}

\section{Introduction}
Structured generation, or mapping natural language into formal representations such as tables, charts, or executable programs, has gained increasing importance with the rise of large language models \cite{yang2019conditional, wu2024protein, xiang2025structured}. This problem is typically formulated by existing methods as a deterministic mapping from input to a single structured output \cite{o2025deterministic, gianniotis2008visualization}. However, this formulation is fundamentally misaligned with the nature of natural language, which is often ambiguous and underspecified, admitting multiple structurally distinct yet semantically valid interpretations \cite{hu2022make, keyvan2022approach}. For example, the sentence "Sales increased in summer and declined in winter" can be represented at different temporal granularities, each yielding a valid structured representation. Despite this, existing approaches rely on single-reference supervision and likelihood-based objectives that collapse multiple plausible interpretations into a single prediction. When ambiguity is present, this mismatch leads to unstable behavior, sensitivity to decoding strategies, and inconsistent outputs. Crucially, the limitation is not merely due to model capacity, but arises from the underlying learning formulation itself: treating structured generation as a single-output prediction problem inherently suppresses valid alternatives and fails to capture the diversity of correct solutions \cite{liu2023ambiguity, di2025design, keluskar2024llms}.

\begin{figure}
    \centering
    \includegraphics[width=0.99\linewidth]{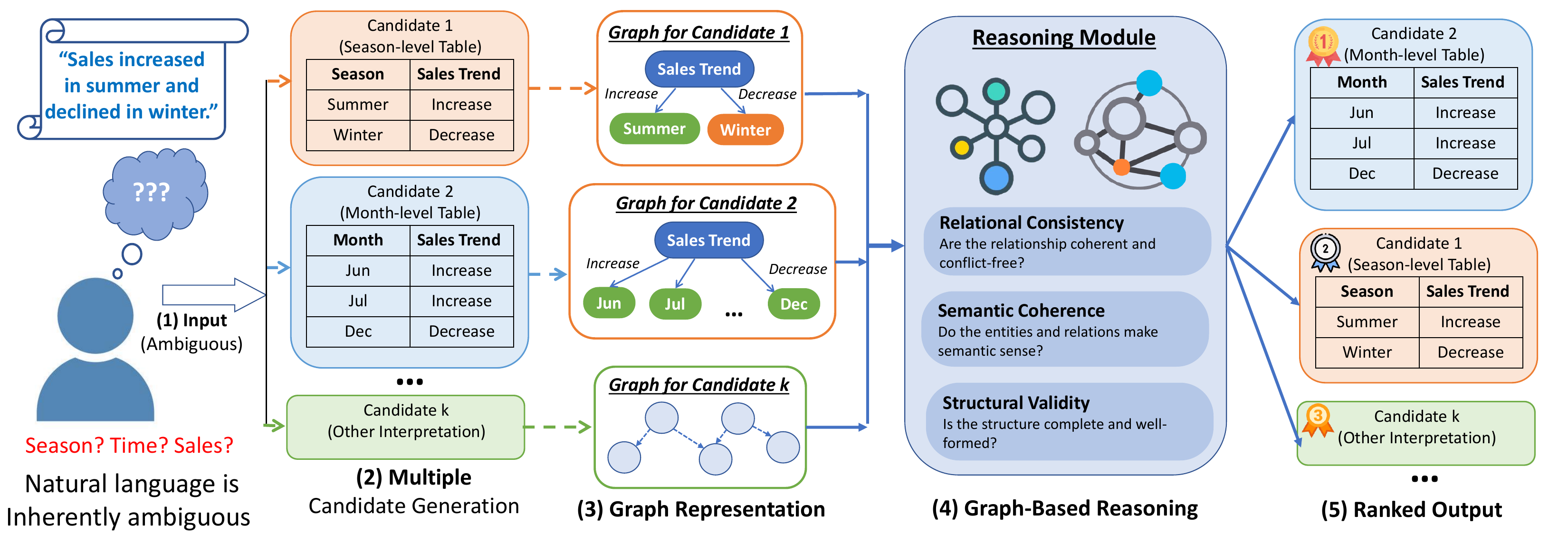}
    \vspace{-0.5em}
\caption{\footnotesize
Overview of the proposed ambiguity-aware structured generation framework.
Given an ambiguous natural-language input, the model first generates multiple
candidate structured representations capturing different plausible
interpretations. Each candidate is then converted into a graph encoding
task-specific entities and their relationships. A graph-based reasoning module
evaluates and compares the candidates based on semantic coherence, structural
validity, and relational consistency, enabling principled selection beyond
token-level likelihoods. Finally, the candidates are ranked and the
highest-scoring structure is selected as the output.}
    \label{fig:overview}
    \vspace{-1em}
\end{figure}

Recent advances in sequence modeling and reasoning have significantly improved structured generation \cite{huguet2024sequence, zheng2024sglang}. Performance on tasks such as text-to-SQL \cite{hong2025next, wang2020rat, lin2020bridging, scholak2021picard, li2023resdsql, pourreza2023din}, table generation \cite{xiang2025structured}, and visualization synthesis \cite{wu2023imagine} has benefited from large language models (LLMs), improved prompting strategies such as chain-of-thought and program-of-thought reasoning, and retrieval-augmented techniques. Concurrently, graph-based methods \cite{cheng2022crystal, li2022sgtr, martinkus2022spectre, wang2025fdgen} model entities and their relationships, enabling more effective reasoning within individual structures. 
Despite these advances, both lines of work share a common limitation: they retain a single-output assumption, treating structured generation as selecting one correct structure per input \cite{zhu2026graph}. LLM-based methods improve surface-level reasoning but still rely on single-reference supervision, while graph-based approaches enhance structural modeling but operate within individual candidates. Neither line of work explicitly models the space of alternative valid interpretations as a first-class component of structured generation. As a result, existing methods lack mechanisms to compare and resolve ambiguity across competing structures \cite{liu2026lidar, szymanski2025limitations}. This limitation is particularly pronounced in underspecified scenarios, where multiple interpretations are equally plausible. Addressing this challenge requires moving beyond deterministic prediction toward frameworks that reason over sets of candidate structures in a principled manner.

In this paper, we reformulate structured generation as an
\textit{ambiguity-aware inference problem}, moving beyond single-output
prediction to reasoning over a set of candidate structured representations
that capture multiple plausible interpretations. We adopt a multi-valid
learning perspective that preserves multiple correct candidates during
training, addressing the limitations of single-reference supervision.
Figure~\ref{fig:overview} illustrates the overall framework: given an
ambiguous input, the model first generates multiple candidate structures,
maps each candidate to a task-specific relational graph, applies graph-based
reasoning to evaluate semantic coherence, structural validity, and relational
consistency, and finally ranks the candidates to select the output.
Importantly, graph construction is task-specific at the schema level, whereas
the subsequent graph encoding, scoring, and ambiguity-aware selection
pipeline follow the same formulation across tasks. This formulation departs
from conventional reranking pipelines by treating ambiguity as a first-class
modeling objective rather than a post-hoc selection problem. The primary
contributions of this paper are summarized as follows.

\begin{itemize}
\item We identify a fundamental limitation of existing approaches, which treat structured generation as a one-to-one deterministic mapping. We show that this formulation is inherently misaligned with natural language, where multiple structured interpretations may be equally valid, and formalize the problem as an \textit{ambiguity-aware inference task}.

\item We introduce an ambiguity-aware framework that explicitly constructs a candidate space of structured representations and adopts a multi-valid learning perspective, preserving multiple correct interpretations during training and avoiding the suppression effects of single-reference supervision.

\item We represent each candidate as a graph encoding entities, attributes, and their relationships, and develop a graph-based reasoning mechanism that evaluates and ranks candidates based on relational consistency, structural validity, and semantic coherence, enabling principled comparison across competing interpretations beyond token-level likelihoods.

\item We demonstrate through extensive experiments that explicitly modeling
and resolving ambiguity leads to more accurate, stable, and robust predictions
across structured generation tasks, with particularly strong gains under high
ambiguity. Controlled ablations further disentangle the contributions of
graph-based reasoning and multi-valid supervision, showing that both
independently improve candidate selection.
\end{itemize}

\section{Methodology}
\label{sec:method}

\noindent{\textbf{\textit{Overview:}}}
We formulate structured generation as an ambiguity-aware inference problem. Given an input query $x$, the framework does not directly commit to a single output. Instead, it constructs a candidate set $\mathcal{S}(x)=\{s_1,\ldots,s_k\}$, where $k$ is the number of generated candidates, converts each candidate $s_i$ into a graph $G_{s_i}$, reasons over candidate graphs, and selects the most coherent structure. The overall pipeline is: input text $\rightarrow$ candidate structures $\rightarrow$ graph construction $\rightarrow$ graph-based reasoning $\rightarrow$ scoring and selection. This design allows the model to compare multiple plausible interpretations before committing to one structured output.

\noindent{\textbf{\textit{Problem Formulation:}}}
Let $x$ be a natural-language input and let $\mathcal{S}(x)=\{s_1,\ldots,s_k\}$ denote the set of plausible structured outputs generated for $x$. Unlike deterministic structured generation, which assumes a one-to-one mapping $x\mapsto s$, we allow $\mathcal{S}(x)$ to contain multiple semantically valid but structurally different candidates. The objective is to select
\begin{equation}
    s^*=\arg\max_{s\in\mathcal{S}(x)} \mathrm{Score}(s,x),
\end{equation}
where $\mathrm{Score}(s,x)$ measures both input alignment and internal structural coherence.

\noindent{\textbf{\textit{Candidate Structure Generation:}}}
For each input $x$, a pretrained language model or task-specific structured generator produces top-$k$ candidates $\mathcal{S}(x)=\{s_1,\ldots,s_k\}$. Candidate diversity is encouraged through stochastic decoding, beam search, or prompting variants. The purpose of this stage is not only to generate high-likelihood outputs, but also to expose alternative interpretations of an underspecified input. In text-to-visualization, candidates may differ in chart type, aggregation, encoding channel, filtering, or temporal grouping. In text-to-SQL, they may differ in join paths, selected columns, conditions, operators, aggregation functions, or equivalent query forms.

\noindent{\textbf{\textit{Graph Construction:}}}
Each candidate $s\in\mathcal{S}(x)$ is first mapped to a typed graph
$G_s=(V_s,E_s)$ using the structural schema of the target task.
Graph construction is task-specific at the schema level and precedes graph
encoding; the GNN does not generate the graphs. For text-to-visualization,
nodes include attributes, chart types, encoding channels, aggregation
functions, filters, grouping fields, and values, while edges encode relations
such as attribute-to-encoding, aggregation-to-attribute, and grouping
dependencies. For text-to-SQL, nodes include tables, columns, operators,
conditions, values, and aggregation functions, while edges encode relations
such as table-to-column, joins, condition dependencies, operator-to-argument,
and condition-to-value links.

This task specificity is confined to mapping each task's native structured
representation into a typed relational graph. Once $G_s$ is constructed, the
same downstream reasoning formulation---edge-aware message passing,
graph-level pooling, candidate scoring, and ambiguity-aware selection---is
applied. Thus, adapting the framework to a new structured-generation task
requires defining an appropriate node and relation schema, but does not
require redesigning the graph-reasoning mechanism or learning objective.
The resulting graph representation makes the internal organization of each
candidate explicit, allowing candidates that are similar at the token level
but differ in grouping, relation choice, or value attachment to be
distinguished structurally.

\noindent{\textbf{\textit{Graph-Based Reasoning:}}}
Given $G_s$, we compute node and graph representations using an edge-aware graph neural network. Let $h_v^{(0)}$ be the initial embedding of node $v\in V_s$, obtained from its node type and textual content when available. At layer $\ell$, node representations are updated as
\begin{equation}
    h_v^{(\ell+1)}
    =
    \sigma\left(
    W_1 h_v^{(\ell)}
    +
    W_2
    \sum_{u\in\mathcal{N}(v)}
    \psi(h_u^{(\ell)},e_{uv})
    \right),
\end{equation}
where $\mathcal{N}(v)$ is the neighborhood of $v$, $e_{uv}$ denotes the edge type between nodes $u$ and $v$, $\psi(\cdot)$ is an edge-aware message function, $\sigma(\cdot)$ is a nonlinear activation, and $W_1,W_2$ are trainable parameters. After $L$ layers, a graph-level representation is obtained as
\begin{equation}
    h_s=\mathrm{Readout}\left(\{h_v^{(L)}:v\in V_s\}\right),
\end{equation}
where $\mathrm{Readout}(\cdot)$ is a permutation-invariant pooling function. The initial node representation combines node-type information with a
contextual representation of the node's associated textual content, when
available, while edge-type information is incorporated through the
edge-aware message function $\psi(\cdot)$. In our implementation, we use
$L=2$ message-passing layers followed by mean pooling to obtain $h_s$.
The graph encoder and scoring module are optimized jointly using the
corresponding training objective described below. Because nvBench and Spider
use different graph schemas, the scorer is trained separately for each
benchmark; parameters are not shared across the two tasks. The GNN
architecture, scoring formulation, and learning objective, however, remain
the same across benchmarks.

\noindent{\textbf{\textit{Ambiguity-Aware Scoring:}}}
Each candidate is scored using its graph representation $h_s$ and the input representation $h_x$, where $h_x$ is obtained by encoding the input query $x$ with the same language encoder or a task-specific text encoder. We define $\mathrm{Score}(s,x)=f(h_s,h_x)$, where $f(\cdot)$ is a trainable scoring function. To separate the main sources of evidence, we use
\begin{equation}
    \mathrm{Score}(s,x)
    =
    \alpha S_{\mathrm{sem}}(s,x)
    +
    \beta S_{\mathrm{str}}(s)
    +
    \gamma S_{\mathrm{rel}}(s,x),
\end{equation}
where $\alpha,\beta,\gamma\geq 0$. The semantic score $S_{\mathrm{sem}}$ measures whether the candidate is supported by the input text. The structural score $S_{\mathrm{str}}$ measures whether the candidate graph is complete and well formed for the target formalism. The relational score $S_{\mathrm{rel}}$ measures whether relations implied by the input are consistent with relations encoded in $G_s$. In implementation, these terms are realized as learned scoring heads over $h_x$, $h_s$, and their interaction. This design compares candidates using semantic, structural, and relational evidence rather than token likelihood alone.

\noindent{\textbf{\textit{Selection and Ranking:}}}
At inference time, the final output is selected as $s^*=\arg\max_{s\in\mathcal{S}(x)}\mathrm{Score}(s,x)$. The same scoring function can also produce a ranked candidate list, which is useful when downstream systems require top-$k$ alternatives or when multiple plausible interpretations should be preserved.

\noindent{\textbf{\textit{Training Objective:}}}
When a single reference structure is available, the model is trained with a supervised ranking loss,
\begin{equation}
\label{eq:single}
    \mathcal{L}_{\mathrm{single}}
    =
    -\log
    \frac{
    \exp(\mathrm{Score}(s_{\mathrm{true}},x))
    }{
    \sum_{s\in\mathcal{S}(x)}
    \exp(\mathrm{Score}(s,x))
    }.
\end{equation}
Here, $s_{\mathrm{true}}$ denotes the reference structure.This objective also provides a single-reference instantiation of the graph reasoning framework when reliable alternative-validity information is not available. For ambiguous inputs, multiple candidates may be valid. Let $\mathcal{S}_{\mathrm{valid}}(x)\subseteq\mathcal{S}(x)$ denote the subset of candidates that are semantically valid for $x$. Depending on the benchmark, this subset can be obtained from annotated alternatives, execution-equivalence checks, visualization-equivalence checks, or human verification of generated candidates. We then optimize
\begin{equation} \label{eq:multivalid}
    \mathcal{L}_{\mathrm{multi}}
    =
    -\log
    \sum_{s\in\mathcal{S}_{\mathrm{valid}}(x)}
    \frac{
    \exp(\mathrm{Score}(s,x))
    }{
    \sum_{\tilde{s}\in\mathcal{S}(x)}
    \exp(\mathrm{Score}(\tilde{s},x))
    }.
\end{equation}
This objective assigns probability mass to all valid interpretations instead of forcing the model toward a single reference. It therefore preserves ambiguity during training while still penalizing unsupported or structurally inconsistent candidates.

\noindent{\textbf{\textit{Discussion:}}}
The proposed method differs from deterministic structured generation in two
ways. First, it explicitly represents ambiguity by constructing a candidate
set instead of committing to a single output early. Second, it evaluates
candidates through graph-based relational reasoning rather than relying only
on sequence likelihood or surface similarity. The graph-reasoning component
can be used with either single-reference or multi-valid supervision, whereas
the additional benefit of multi-valid training requires reliable identification
of valid alternatives.

The framework is applicable when structured outputs admit an explicit or
reliably recoverable relational representation. For example, program
representations may be expressed using abstract syntax, control-flow, or
data-flow relations, while knowledge-graph outputs directly expose entities
and relations. In contrast, the present framework is not directly applicable
to unconstrained free-form generation when no stable relational decomposition
or reliable semantic-equivalence criterion exists; such settings would require
an additional graph-induction or representation-learning mechanism. Thus, the
general contribution is ambiguity-aware reasoning over task-appropriate
relational representations rather than a universal graph parser.

\section{Theoretical Analysis}
\label{sec:theory}

We analyze ambiguity-aware structured generation under the candidate-space formulation introduced in Section~\ref{sec:method}. The analysis separates three roles of the framework: candidate generation provides coverage, multi-valid training avoids suppressing valid alternatives, and graph-relational scoring increases valid-invalid separation. All proofs are deferred to Appendix~\ref{app:theory_proofs}.

\noindent{\textbf{\textit{Ambiguous structured output space:}}}
Let $x\in\mathcal{X}$ denote an input query and let $\mathcal{Z}$ be the space of structured outputs. The set of all semantically valid outputs for $x$ is $\mathcal{Y}(x)\subseteq\mathcal{Z}$. Deterministic structured generation effectively treats the supervision signal as if $|\mathcal{Y}(x)|=1$, while ambiguous inputs satisfy $|\mathcal{Y}(x)|>1$. A generator $q_{\phi}(s|x)$ produces a candidate set $\mathcal{S}(x)=\{s_1,\ldots,s_k\}$, and the valid candidate subset is
$
    \mathcal{S}_{\mathrm{valid}}(x)
    =
    \mathcal{S}(x)\cap\mathcal{Y}(x).
$
Coverage holds when $\mathcal{S}_{\mathrm{valid}}(x)\neq\emptyset$. This condition is necessary for valid selection, but it does not by itself determine which candidate is selected.

\begin{lemma}[Candidate coverage]
\label{lem:candidate_coverage}
Assume candidates are independently sampled from $q_{\phi}(s|x)$, and let $p_{\mathrm{valid}}(x)=q_{\phi}(\mathcal{Y}(x)|x)$ be the probability that one sampled candidate is valid. Then
\begin{equation}
    \Pr\!\left[\mathcal{S}_{\mathrm{valid}}(x)\neq\emptyset\right]
    =
    1-\left(1-p_{\mathrm{valid}}(x)\right)^k .
\end{equation}
\end{lemma}
The proof is provided in Appendix~\ref{app:proof_candidate_coverage}.
Lemma~\ref{lem:candidate_coverage} shows that increasing $k$ improves the probability that at least one valid structure appears in the candidate set. It does not guarantee correct selection, which still depends on the scoring function.

\noindent{\textbf{\textit{Single-reference and multi-valid objectives:}}}
For each $s\in\mathcal{S}(x)$, let $z_s=\mathrm{Score}(s,x)$ and define the candidate softmax
$
    P_{\theta}(s|x)
    =
    \frac{\exp(z_s)}
    {\sum_{\tilde{s}\in\mathcal{S}(x)}\exp(z_{\tilde{s}})}.
$
With a single annotated reference $s_{\mathrm{ref}}\in\mathcal{Y}(x)$, the standard objective is $\mathcal{L}_{\mathrm{single}}=-\log P_{\theta}(s_{\mathrm{ref}}|x)$. Under ambiguity, we instead define the valid-set probability $Q_{\theta}(x)=\sum_{s\in\mathcal{S}_{\mathrm{valid}}(x)}P_{\theta}(s|x)$ and optimize $\mathcal{L}_{\mathrm{multi}}=-\log Q_{\theta}(x)$.

\begin{proposition}[Single-reference gradients suppress valid alternatives]
\label{prop:single_ref_misalignment}
Assume $|\mathcal{Y}(x)|>1$ and let $s'\in\mathcal{S}_{\mathrm{valid}}(x)$ with $s'\neq s_{\mathrm{ref}}$ . Under $\mathcal{L}_{\mathrm{single}}$,
$
    \frac{\partial \mathcal{L}_{\mathrm{single}}}{\partial z_{s'}}
    =
    P_{\theta}(s'|x)>0 .
$
Thus, gradient descent decreases the score of $s'$, although $s'$ is semantically valid.
\end{proposition}
The proof is provided in Appendix~\ref{app:proof_single_ref}.
Proposition~\ref{prop:single_ref_misalignment} formalizes the mismatch created by single-reference cross-entropy. Every non-reference candidate is treated as a negative, including valid alternatives.

\begin{proposition}[Multi-valid gradients preserve valid alternatives]
\label{prop:multi_valid_gradient}
For $\mathcal{L}_{\mathrm{multi}}=-\log Q_{\theta}(x)$, the gradient with respect to $z_s$ is
\begin{equation}
    \frac{\partial \mathcal{L}_{\mathrm{multi}}}{\partial z_s}
    =
    \begin{cases}
    P_{\theta}(s|x)\left(1-\frac{1}{Q_{\theta}(x)}\right),
    & s\in\mathcal{S}_{\mathrm{valid}}(x),\\[0.5em]
    P_{\theta}(s|x),
    & s\notin\mathcal{S}_{\mathrm{valid}}(x).
    \end{cases}
\end{equation}
Hence, valid candidates receive non-positive gradients, while invalid candidates receive positive gradients. Moreover, if $s_{\mathrm{ref}}\in\mathcal{S}_{\mathrm{valid}}(x)$, then $\mathcal{L}_{\mathrm{multi}}\leq\mathcal{L}_{\mathrm{single}}$.
\end{proposition}
The proof is provided in Appendix~\ref{app:proof_multi_valid_gradient}.
Proposition~\ref{prop:multi_valid_gradient} shows that the multi-valid objective increases probability mass over the valid candidate set rather than concentrating supervision on one reference.

\noindent{\textbf{\textit{Valid-structure recovery:}}}
Let $U(s,x)$ be an ideal utility that ranks semantically valid and structurally coherent candidates above invalid ones. The learned scorer is
$
    \widehat U(s,x)
    =
    \mathrm{Score}(s,x)
    =
    \alpha S_{\mathrm{sem}}(s,x)
    +
    \beta S_{\mathrm{str}}(s)
    +
    \gamma S_{\mathrm{rel}}(s,x),
$
where the three terms correspond to semantic support, structural validity, and graph-relational agreement. If $\mathcal{S}(x)\setminus\mathcal{S}_{\mathrm{valid}}(x)=\emptyset$, the result is immediate. Otherwise, define the ideal valid-invalid margin
$
    \Delta(x)
    =
    \min_{s^+\in\mathcal{S}_{\mathrm{valid}}(x)}
    U(s^+,x)
    -
    \max_{s^-\in\mathcal{S}(x)\setminus\mathcal{S}_{\mathrm{valid}}(x)}
    U(s^-,x).
$

\begin{theorem}[Valid-structure recovery under coverage and score accuracy]
\label{thm:valid_recovery}
Suppose $\mathcal{S}_{\mathrm{valid}}(x)\neq\emptyset$, $\Delta(x)>0$, and $|\widehat U(s,x)-U(s,x)|\leq\epsilon$ for all $s\in\mathcal{S}(x)$. If $\epsilon<\Delta(x)/2$, then
$
    \hat{s}
    =
    \arg\max_{s\in\mathcal{S}(x)}
    \widehat U(s,x)
$
satisfies $\hat{s}\in\mathcal{S}_{\mathrm{valid}}(x)$.
\end{theorem}
 See Appendix~\ref{app:proof_valid_recovery} for the proof.
Theorem~\ref{thm:valid_recovery} states that valid recovery requires both coverage and sufficiently accurate score estimation. In the proposed method, the scoring decomposition is useful when its semantic, structural, and graph-relational terms enlarge the valid-invalid margin or reduce scoring error.

\noindent{\textbf{\textit{Graph-relational margin:}}}
Each candidate $s$ is mapped to a graph $G_s=(V_s,E_s)$. Let $\varphi_G(G_s)$ be a graph representation of the candidate and let $\varphi_X(x)$ be an input-side relational representation. Define the relational mismatch
$
    d_{\mathrm{rel}}(s,x)
    =
    \|\varphi_G(G_s)-\varphi_X(x)\|_2,
    \enspace
    S_{\mathrm{rel}}(s,x)=-d_{\mathrm{rel}}(s,x).
$
This term is intended to capture relational ambiguity, including joins, grouping, condition attachment, attribute-value links, and aggregation structure.

\begin{proposition}[Graph-relational scoring increases valid-invalid separation]
\label{prop:relational_margin}
Assume that valid candidates are relationally closer to the input than invalid candidates by margin $\rho>0$, that is,
\begin{equation}
    d_{\mathrm{rel}}(s^-,x)-d_{\mathrm{rel}}(s^+,x)\geq\rho,
    \qquad
    \forall s^+\in\mathcal{S}_{\mathrm{valid}}(x),\;
    s^-\in\mathcal{S}(x)\setminus\mathcal{S}_{\mathrm{valid}}(x).
\end{equation}
Then adding $\gamma S_{\mathrm{rel}}(s,x)$ to the score increases the valid-invalid score gap by at least $\gamma\rho$.
\end{proposition}
The proof is given in Appendix~\ref{app:proof_relational_margin}.
Proposition~\ref{prop:relational_margin} ties graph reasoning directly to ambiguity-aware selection. When candidates are lexically similar but differ in relational structure, the graph-relational term provides an additional margin for selection.

\begin{corollary}[Selection stability under score perturbations]
\label{cor:selection_stability}
Let $\hat{s}=\arg\max_{s\in\mathcal{S}(x)}\widehat U(s,x)$ and define the learned score margin
$
    m(x)
    =
    \widehat U(\hat{s},x)
    -
    \max_{s\neq\hat{s}}
    \widehat U(s,x).
$
If perturbed scores $\widetilde U(s,x)$ satisfy $|\widetilde U(s,x)-\widehat U(s,x)|\leq\eta$ for all $s\in\mathcal{S}(x)$ and $\eta<m(x)/2$, then $\hat{s}$ remains the selected candidate.
\end{corollary}
See Appendix~\ref{app:proof_selection_stability} for the proof.
Corollary~\ref{cor:selection_stability} frames stability under decoding randomness, scorer perturbations, or small candidate-score noise. Larger learned margins reduce output instability under ambiguity.

\textbf{Remark 1.} Propositions~\ref{prop:single_ref_misalignment} and~\ref{prop:multi_valid_gradient} characterize the training mismatch caused by single-reference supervision. When several candidates are valid, the single-reference loss pushes down every non-reference candidate, whereas the multi-valid objective optimizes the probability of the valid set. This motivates evaluating predictions against multiple valid references rather than a single annotated target, as reported in Table~\ref{tab:multi_ref_results}.

\textbf{Remark 2.} Lemma~\ref{lem:candidate_coverage}, Theorem~\ref{thm:valid_recovery}, Proposition~\ref{prop:relational_margin}, and Corollary~\ref{cor:selection_stability} identify the roles of generation, scoring, and graph reasoning. Candidate generation controls whether a valid structure is present, score accuracy determines whether it can be recovered, and graph-relational evidence increases the margin needed for stable selection. This links the theory to the candidate-size, graph-ablation, graph-separability, and consistency analyses in Tables~\ref{tab:ablation_k}, \ref{tab:ablation_graph}, \ref{tab:graph_separability}, and~\ref{tab:consistency}.

\section{Experiments}

\subsection{Experimental Setup}
We evaluate our framework through four research questions (RQs):

\textbf{RQ1:} Does ambiguity-aware modeling improve structured generation over deterministic approaches? \\
\textbf{RQ2:} Does graph-based reasoning improve selection beyond candidate generation and reranking? \\
\textbf{RQ3:} How does performance vary across different levels of input ambiguity? \\
\textbf{RQ4:} Which components contribute most to performance?

\textbf{Implementation Details.}
We use LLaMA-2-13B~\cite{touvron2023llama} for stochastic decoding
with temperature sampling ($T{=}0.7$) and generate $k{=}5$ candidate
outputs per input. Each candidate is converted into a task-specific
relational graph and encoded using an edge-aware GNN with $L{=}2$
message-passing layers followed by mean pooling. The GNN encoder and
candidate-scoring module are optimized jointly using the multi-valid
objective with Adam~\cite{kingma2014adam}, a learning rate of
$1{\times}10^{-4}$, batch size 32, and early stopping. Because nvBench
and Spider use different graph schemas, the scorer is trained separately
for each benchmark, while the GNN architecture, scoring formulation,
and learning objective remain unchanged. Results are averaged over three
random seeds. Additional implementation details are provided in
Appendix~C.

\textbf{Baselines.}
We compare against representative deterministic and candidate-based
methods. Under the standard one-to-one formulation,
\textbf{Deterministic LLM} directly generates a single structured output
per input. \textbf{Candidate Only} generates the same candidate pool used
by the candidate-based methods but selects the first candidate in the
generator's native generation order, without any additional scoring or
reranking. \textbf{Candidate + Likelihood} instead ranks the generated
candidates according to their generation likelihood.
\textbf{Candidate + LLM Rerank} applies an additional LLM-based semantic
scoring pass to rerank the same candidate pool. These baselines therefore
separate the effects of candidate generation, likelihood-based selection,
and additional semantic reranking from the proposed graph-based reasoning
module (see Appendix~\ref{app:metrics_baselines} for details).

\subsection{Datasets}
\begin{wraptable}{r}{0.55\linewidth}
\centering
\footnotesize
\vspace{-2em}
\caption{\small Dataset statistics with ambiguity analysis. Inputs are categorized into Low, Medium, and High ambiguity based on the number of valid structured interpretations.}
\label{tab:dataset_stats}
\renewcommand{\arraystretch}{1.2}
\setlength{\tabcolsep}{2pt}
\begin{tabular}{l|c|c|c|c|c|c}
\hline
\textbf{Dataset} & \textbf{\#Samples} & \textbf{Avg Len} & \textbf{\#Valid Outputs} & \textbf{Low} & \textbf{Med} & \textbf{High} \\
\hline
nvBench & 7,878 & 14.2 & 2.6 & 25\% & 33\% & 42\% \\
Spider  & 10,181 & 18.7 & 1.8 & 40\% & 35\% & 25\% \\
\hline
\end{tabular}
\vspace{-1em}
\end{wraptable}

We evaluate on structured generation benchmarks with inherent ambiguity. \textbf{nvBench}~\cite{luo2025nvbench}. A text-to-visualization dataset where, depending on aggregation, encoding, and grouping decisions, a query may correspond to multiple valid specifications. \textbf{Spider}~\cite{yu2018spider}. A large-scale text-to-SQL benchmark that introduces implicit ambiguity by allowing multiple SQL queries to be execution-equivalent.

\textbf{Ambiguity Annotation.}
We construct a candidate set for each input and determine valid outputs based on semantic equivalence. For nvBench, candidates are valid if they produce comparable visual encodings; for Spider, validity is based on execution equivalence. To analyze ambiguity, we categorize inputs by the number of valid candidates: \textit{low} (1), \textit{medium} (2--3), and \textit{high} ($\geq$4).

\subsection{Main Results}
Table~\ref{tab:main_results} answers \textbf{RQ1} by comparing our approach with candidate-based and deterministic baselines. Across both datasets, our method consistently outperforms all baselines, demonstrating the effectiveness of ambiguity-aware modeling. Candidate generation improves performance by exposing multiple plausible interpretations, while deterministic models degrade under ambiguity due to the one-to-one assumption. However, likelihood-based selection and LLM reranking remain limited, relying on surface-level plausibility rather than relational reasoning. By modeling relational consistency, our approach achieves the strongest performance. The gap between Top-$k$ and EM shows that candidate generation alone is insufficient, while Top-$k$ provides an approximate upper bound given the candidate pool. Gains are more pronounced on nvBench, which exhibits higher ambiguity.
\begin{wraptable}{r}{0.45\linewidth}
\centering
\footnotesize
\vspace{-1em}
\caption{\small Main results on structured generation tasks. We report Exact Match (EM) and Top-$k$ Success. Best results are shown in bold.}
\label{tab:main_results}
\renewcommand{\arraystretch}{1.2}
\setlength{\tabcolsep}{3pt}
\begin{tabular}{l|cc|cc}
\hline
\textbf{Method} & \multicolumn{2}{c|}{\textbf{nvBench}} & \multicolumn{2}{c}{\textbf{Spider}} \\
 & EM & Top-5 & EM & Top-5 \\
\hline
Deterministic (LLM) & 48.3 & -- & 62.5 & -- \\
Candidate + Likelihood & 52.1 & 71.4 & 65.2 & 79.8 \\
Candidate + LLM Rerank & 54.7 & 74.3 & 66.8 & 82.1 \\
\hline
\textbf{Ours} & \textbf{58.9} & \textbf{81.6} & \textbf{70.3} & \textbf{86.5} \\
\hline
\end{tabular}
\vspace{-5em}
\end{wraptable}

\subsection{Analysis on Ambiguity}
This section addresses \textbf{RQ3} by examining how performance varies across different levels and types of ambiguity.

\textbf{Performance under Different Ambiguity Levels.}
Based on the number and diversity of plausible interpretations, we categorize inputs into low, medium, and high ambiguity. In low-ambiguity settings, where a dominant structure exists, all methods perform comparably, as shown in Figure~\ref{fig:ambiguity_levels}. However, as ambiguity increases, deterministic and candidate-based approaches degrade significantly. In contrast, our approach demonstrates stronger robustness under ambiguous inputs by maintaining a smaller performance drop and achieving the largest gains in high-ambiguity cases.

\textbf{Breakdown by Ambiguity Type.}
We further analyze performance across temporal, aggregation, and relational ambiguity. Deterministic models perform poorly across all types, with the largest drops observed in relational and aggregation cases, as shown in Figure~\ref{fig:ambiguity_types}. While candidate-based approaches provide modest improvements, they remain limited in resolving structural inconsistencies. In relational ambiguity, where modeling entity relationships is critical, our approach consistently outperforms these baselines.

\begin{figure}[t]
    \centering
    \begin{minipage}{0.46\linewidth}
        \centering
        \includegraphics[width=\linewidth]{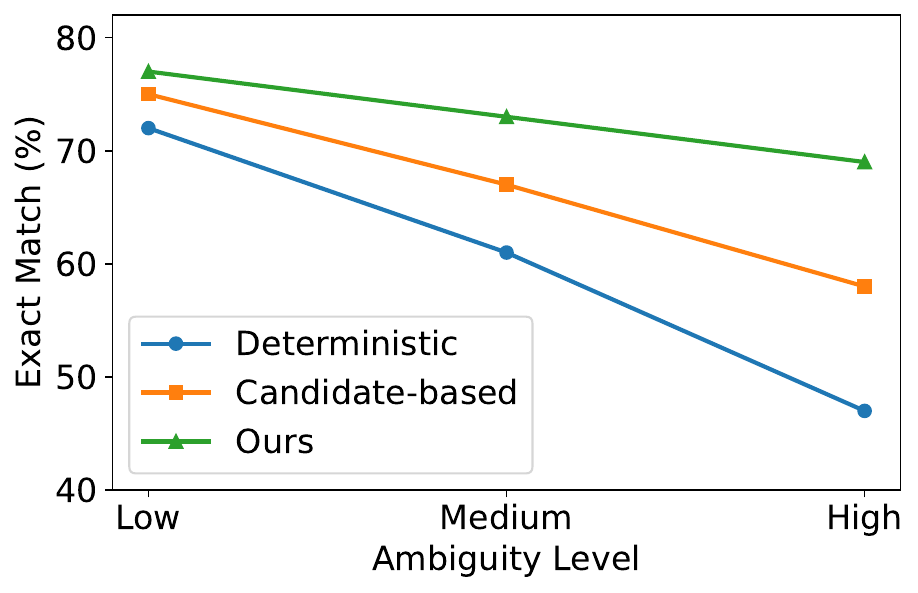}
        \vspace{-2em}
        \caption{Performance under different ambiguity levels.}
        \label{fig:ambiguity_levels}
    \end{minipage}
    \hfill
    \begin{minipage}{0.49\linewidth}
        \centering
        \includegraphics[width=\linewidth]{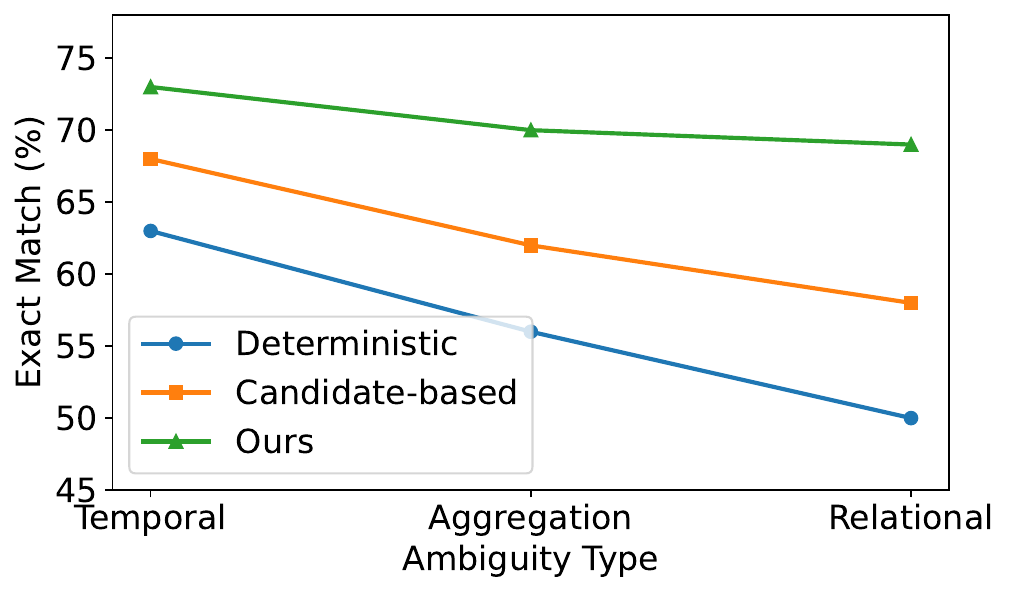}
        \vspace{-2em}
        \caption{Performance across different ambiguity types.}
        \label{fig:ambiguity_types}
    \end{minipage}
    \vspace{-1em}
\end{figure}
\subsection{Evaluation under Multiple Valid Outputs}
\label{sec:multi_valid}

\textbf{Problem Setup.}
This section addresses \textbf{RQ1} and \textbf{RQ3} by evaluating models' ability to handle multiple valid structured outputs rather than a single reference. Standard evaluation relies on a single reference and penalizes semantically correct alternatives, while natural language often admits multiple valid interpretations. To align with our ambiguity-aware formulation, we consider multi-reference settings. For ambiguous inputs, we construct a set of valid structures $\mathcal{S}_{\text{valid}}(x)$ from annotated alternatives or verified high-quality candidates.

\textbf{Multi-Reference Evaluation.}
We extend exact match (EM) to account for multiple references:
\(
\text{Multi-EM} = \max_{s \in \mathcal{S}_{\text{valid}}(x)} \mathbb{1}(\hat{s} = s).
\)
This metric provides a more accurate assessment under ambiguity by rewarding predictions that match any valid structure. Deterministic baselines show limited gains, as illustrated in Table~\ref{tab:multi_ref_results}, whereas our method achieves substantially higher Multi-EM, indicating improved selection among valid alternatives.

\textbf{Consistency under Ambiguity.}
We further evaluate prediction stability by measuring agreement across multiple stochastic outputs. All methods remain consistent in low-ambiguity settings but degrade as ambiguity increases, as shown in Table~\ref{tab:consistency}. In contrast, our approach maintains significantly higher consistency, particularly under high ambiguity, suggesting that graph-based reasoning stabilizes selection among competing interpretations.

\subsection{Ablation Studies}
This section addresses \textbf{RQ2} and \textbf{RQ4} by isolating the impact of candidate set size, scoring components, and graph-based reasoning. The largest improvements arise from graph-based relational reasoning, followed by the combination of structural and semantic scoring.

\textbf{Effect of Graph Reasoning and Multi-Valid Supervision.}
To disentangle the contribution of graph-based reasoning from that of
multi-valid supervision, we additionally evaluate a
\textbf{Graph + Single-Reference} variant. This variant uses the same
candidate sets, graph construction, two-layer GNN, scoring architecture,
optimization settings, and data splits as the full model, but replaces
the multi-valid objective in Eq.~\ref{eq:multivalid} with the
single-reference objective in Eq.~\ref{eq:single}.
As shown in Table~\ref{tab:ablation_graph}, graph-based scoring improves
over LLM reranking by 2.4 EM points on nvBench and 2.0 points on Spider
under the same single-reference setting. Multi-valid supervision provides
a further improvement of 1.8 and 1.5 points, respectively. These results
show that graph-based reasoning and multi-valid supervision provide
complementary gains.

\textbf{Effect of Candidate Set Size.}
As shown in Table~\ref{tab:ablation_k}, increasing the number of candidates initially improves performance by expanding coverage of valid interpretations. However, baseline methods quickly saturate, indicating limited capacity to leverage additional candidates. In contrast, our approach continues to improve as $k$ increases, suggesting more effective utilization of diverse hypotheses through structured scoring.

\textbf{Effect of Scoring Components.}
Table~\ref{tab:ablation_score} indicates that removing either semantic coherence or structural consistency leads to noticeable performance degradation, demonstrating that both signals are essential. Semantic coherence ensures alignment with the input, while structural consistency captures representation validity. Their combination yields the best performance, emphasizing the importance of jointly modeling relational structure and semantic alignment for reliable candidate selection.

\begin{table*}[t]
\centering

\begin{minipage}{0.49\textwidth}
\centering
\footnotesize
\setlength{\tabcolsep}{2pt}
\caption{\small Multi-reference evaluation results. We report standard Exact Match (EM) and Multi-EM, which accounts for multiple valid structured outputs.}
\label{tab:multi_ref_results}
\begin{tabular}{l|cc|cc}
\hline
\multirow{2}{*}{\textbf{Method}} & \multicolumn{2}{c|}{\textbf{nvBench}} & \multicolumn{2}{c}{\textbf{Spider}} \\
 & EM & Multi-EM & EM & Multi-EM \\
\hline
Deterministic (LLM) & 48.3 & 52.7 & 62.5 & 66.1 \\
Candidate + Likelihood & 52.1 & 58.9 & 65.2 & 70.4 \\
Candidate + LLM Rerank & 54.7 & 61.3 & 66.8 & 72.6 \\
\textbf{Ours} & \textbf{58.9} & \textbf{68.4} & \textbf{70.3} & \textbf{78.2} \\
\hline
\end{tabular}
\vspace{-1em}
\end{minipage}
\hfill
\begin{minipage}{0.49\textwidth}
\centering
\footnotesize
\caption{\small Consistency under different ambiguity levels. Higher values indicate more stable predictions across multiple decoding runs.}
\label{tab:consistency}
\begin{tabular}{l|ccc}
\hline
\textbf{Method} & Low & Medium & High \\
\hline
Deterministic (LLM) & 0.91 & 0.68 & 0.52 \\
Candidate + Likelihood & 0.88 & 0.72 & 0.60 \\
Candidate + LLM Rerank & 0.90 & 0.75 & 0.63 \\
\textbf{Ours} & \textbf{0.93} & \textbf{0.82} & \textbf{0.74} \\
\hline
\end{tabular}
\vspace{-1em}
\end{minipage}

\end{table*}

\begin{table*}[t]
\centering
\begin{minipage}{0.32\textwidth}
\centering
\footnotesize
\setlength{\tabcolsep}{1pt}
\caption{\small Effect of graph reasoning and multi-valid supervision.}

\label{tab:ablation_graph}
\begin{adjustbox}{width=\textwidth}
\begin{tabular}{l|cc}
\hline
\textbf{Method} & \textbf{nvBench EM} & \textbf{Spider EM} \\
\hline
\shortstack{Candidate +\\LLM Rerank}
    & 54.7 & 66.8 \\
    \hline
\shortstack{Graph + \\Single-Reference}
    & $57.1 \pm 0.4$ & $68.8 \pm 0.5$ \\
    \hline
\shortstack{Graph + \\Multi-Valid}
    & $\mathbf{58.9 \pm 0.3}$ & $\mathbf{70.3 \pm 0.4}$ \\
\hline
\end{tabular}
\end{adjustbox}
\vspace{-1em}
\end{minipage}
\hfill
\begin{minipage}{0.23\textwidth}
\centering
\footnotesize
\setlength{\tabcolsep}{2pt}
\caption{\small Effect of candidate set size ($k$).}
\label{tab:ablation_k}
\begin{adjustbox}{width=\textwidth}
\begin{tabular}{c|cc|cc}
\hline
\textbf{$k$} & \multicolumn{2}{c|}{\textbf{nvBench}} & \multicolumn{2}{c}{\textbf{Spider}} \\
 & EM & T5 & EM & T5 \\
\hline
1 & 48.3 & -- & 62.5 & -- \\
3 & 54.7 & 74.3 & 66.8 & 82.1 \\
5 & 58.9 & 81.6 & 70.3 & 86.5 \\
10 & 59.5 & 83.0 & 71.0 & 87.2 \\
\hline
\end{tabular}
\end{adjustbox}
\vspace{-1em}
\end{minipage}
\hfill
\begin{minipage}{0.41\textwidth}
\centering
\setlength{\tabcolsep}{1pt}
\large
\caption{\small Effect of scoring components.}
\label{tab:ablation_score}
\begin{adjustbox}{width=\textwidth}
\begin{tabular}{l|cc|cc}
\hline
\textbf{Variant} & \multicolumn{2}{c|}{\textbf{nvBench}} & \multicolumn{2}{c}{\textbf{Spider}} \\
 & EM & T5 & EM & T5 \\
\hline
Full Model & \textbf{58.9} & \textbf{81.6} & \textbf{70.3} & \textbf{86.5} \\
w/o Struct. Consistency & 55.6 & 76.2 & 67.9 & 82.4 \\
w/o Semantic Coherence & 56.2 & 77.1 & 68.3 & 83.0 \\
\hline
\end{tabular}
\end{adjustbox}
\vspace{-1em}
\end{minipage}

\end{table*}

\begin{table*}[t]
\centering

\begin{minipage}{0.42\textwidth}
\centering
\footnotesize
\setlength{\tabcolsep}{1pt}
\caption{\small Candidate graph separability analysis across datasets. Larger gaps indicate better discrimination.}
\label{tab:graph_separability}
\begin{tabular}{l|cc|cc|cc}
\hline
\textbf{Method} 
& \multicolumn{3}{c|}{\textbf{nvBench}} 
& \multicolumn{3}{c}{\textbf{Spider}} \\
 & Valid & Invalid & Gap & Valid & Invalid & Gap \\
\hline
Can + Likelihood & 0.71 & 0.58 & 0.13 & 0.74 & 0.60 & 0.14 \\
Can + LLM Rerank & 0.76 & 0.57 & 0.19 & 0.78 & 0.58 & 0.20 \\
\textbf{Ours} & \textbf{0.84} & \textbf{0.49} & \textbf{0.35} 
               & \textbf{0.86} & \textbf{0.47} & \textbf{0.39} \\
\hline
\end{tabular}
\end{minipage}
\hfill
\begin{minipage}{0.52\textwidth}
\centering
\footnotesize
\setlength{\tabcolsep}{1pt}
\caption{\small Performance under varying graph complexity across datasets.}
\label{tab:graph_complexity}
\begin{tabular}{l|ccc|ccc}
\hline
\textbf{Method} 
& \multicolumn{3}{c|}{\textbf{nvBench}} 
& \multicolumn{3}{c}{\textbf{Spider}} \\
 & Small & Medium & Large & Small & Medium & Large \\
\hline
Deterministic (LLM) & 61.2 & 52.8 & 43.5 & 68.9 & 59.4 & 49.2 \\
Can + Likelihood & 65.7 & 57.9 & 48.6 & 72.3 & 63.5 & 53.7 \\
Can + LLM Rerank & 67.3 & 59.8 & 50.2 & 74.0 & 65.2 & 55.6 \\
\textbf{Ours} & \textbf{70.5} & \textbf{64.1} & \textbf{57.8} 
               & \textbf{77.6} & \textbf{71.2} & \textbf{64.3} \\
\hline
\end{tabular}
\end{minipage}

\end{table*}

\begin{table}[t]
\centering
\footnotesize
\caption{\small Failure mode analysis. We report common error types, their frequency (\%), and primary causes.}
\label{tab:failure_cases}
\begin{tabular}{l|c|p{7.2cm}}
\hline
\textbf{Failure Type} & \textbf{Freq (\%)} & \textbf{Primary Cause} \\
\hline
Missing entities in candidate graphs & 28.4 & Incomplete candidate generation leading to missing nodes \\
Incorrect relation extraction & 24.7 & Ambiguous phrasing causing incorrect edge construction \\
Overlapping valid interpretations & 19.3 & Multiple correct outputs not fully captured in annotations \\
Noisy candidate generation & 17.6 & Low-quality or semantically inconsistent candidates \\
Close scoring ambiguity & 10.0 & Small score differences between competing valid structures \\
\hline
\end{tabular}
\vspace{-1em}
\end{table}

\subsection{Graph-Centric Analysis} \label{sec:graph_analysis}
We analyze our framework from a graph-centric perspective by examining candidate separability, robustness to graph complexity, and common failure modes.

\textbf{Candidate Graph Separability.}
As shown in Table~\ref{tab:graph_separability}, baseline methods exhibit limited separation between valid and invalid candidates, indicating difficulty in distinguishing structurally inconsistent interpretations. This gap is modestly improved by LLM reranking, but remains constrained by surface-level signals. In contrast, our approach achieves a substantially larger valid--invalid gap across both datasets by explicitly modeling relational consistency rather than token-level similarity, demonstrating that graph-based reasoning provides stronger discrimination.

\textbf{Effect of Graph Complexity.}
Table~\ref{tab:graph_complexity} shows that performance degrades as graph size increases, reflecting the growing difficulty of modeling complex relational structures. This degradation is more pronounced for deterministic and candidate-based baselines, which struggle to maintain consistency in larger structures. In contrast, our approach remains significantly more robust, particularly for large graphs, demonstrating its ability to capture higher-order dependencies through structured representations.

\textbf{Failure Case Analysis.}
As shown in Table~\ref{tab:failure_cases}, most errors arise from limitations in candidate generation, including missing entities and incorrect relation extraction. Additional failures stem from annotation ambiguity and closely scored competing valid structures. These results suggest that while graph-based reasoning improves selection, overall performance remains constrained by candidate quality, indicating that improvements in generation and more precise scoring could further enhance outcomes. These observations are consistent with our theoretical analysis, which shows that improved valid--invalid separation leads to more stable selection.

\textbf{Comparison with SOTA.}
For completeness, we also report performance under the benchmark-specific
settings used for comparison with representative published systems,
including RAT-SQL~\cite{wang2020rat}, BRIDGE~\cite{lin2020bridging},
PICARD~\cite{scholak2021picard}, RESDSQL~\cite{li2023resdsql}, and
DIN-SQL~\cite{pourreza2023din}. These results use the corresponding
benchmark evaluation configuration and should therefore be distinguished
from the controlled experiments in Table~\ref{tab:main_results}, where
generation and candidate-selection conditions are held fixed for
component-level comparison. Table~\ref{tab:sota} reports the resulting
benchmark-level comparison.

\begin{wraptable}{r}{0.34\linewidth}
\vspace{-0.8em}
\centering
\footnotesize
\setlength{\tabcolsep}{2pt}
\caption{\small Comparison with recent SOTA methods on structured generation benchmarks.}
\label{tab:sota}
\begin{tabular}{l|cc|cc}
\hline
\textbf{Method} & \multicolumn{2}{c|}{\textbf{Spider}} & \multicolumn{2}{c}{\textbf{nvBench}} \\
 & EM & Top-5 & EM & Top-5 \\
\hline
PICARD \cite{scholak2021picard} & 75.5 & 84.3 & 68.2 & 79.5 \\
RESDSQL \cite{li2023resdsql} & 79.9 & 88.1 & 72.5 & 83.7 \\
DIN-SQL \cite{pourreza2023din} & 82.1 & 89.7 & 74.8 & 85.9 \\
BRIDGE \cite{lin2020bridging} & 65.5 & 75.2 & 60.3 & 71.4 \\
RAT-SQL \cite{wang2020rat} & 69.7 & 78.6 & 63.9 & 74.8 \\
\hline
\textbf{Ours} & \textbf{83.4} & \textbf{91.2} & \textbf{76.3} & \textbf{88.5} \\
\hline
\end{tabular}
\vspace{-1em}
\end{wraptable}

\textbf{Computational Cost.}
Generating multiple candidates incurs additional cost relative to
single-output decoding and is shared by all candidate-based methods in our
evaluation. Relative to methods operating on the same candidate pool, our
additional computation consists of task-specific graph construction and
two-layer GNN scoring. The reported experimental runtimes refer to
training and evaluation computation rather than per-example inference
latency; we do not report a dedicated latency or throughput benchmark and
therefore make no claim of real-time efficiency.

\section{Limitations}
Although our framework improves structured generation under ambiguity, it remains dependent on candidate quality and coverage, since the selection module cannot recover a valid structure that was not generated. Multi-valid training also requires reliable identification of valid alternatives, while graph construction remains task-specific and may require adaptation for new domains. In addition, generating multiple candidates and applying graph-based reasoning introduces extra computational cost, and we do not evaluate dedicated inference latency or substantially larger or smaller candidate generators. These limitations motivate future work on improved candidate generation, more general graph construction, efficient selection, and reliable approximation of valid alternatives.

\section{Conclusion}
In this paper, we proposed an ambiguity-aware framework for structured
generation that reframes the task as reasoning over multiple candidate
interpretations rather than predicting a single output. The framework
represents candidate structures as task-specific relational graphs and
evaluates them using semantic coherence, structural validity, and relational
consistency. Experiments on text-to-visualization and text-to-SQL benchmarks
demonstrate consistent improvements in accuracy, robustness, and stability,
particularly under high ambiguity. Controlled ablations further show that
graph-based reasoning improves candidate discrimination under single-reference
supervision, while multi-valid training provides additional gains by
preserving alternative valid interpretations. These results highlight the
limitations of deterministic single-output formulations and support explicit
modeling of ambiguity in structured generation. Future work may focus on
improving candidate coverage, reducing task-specific graph-construction
requirements, approximating valid alternatives when explicit equivalence
criteria are unavailable, and developing more efficient selection mechanisms
for scalable deployment.

\bibliographystyle{unsrt}
\bibliography{main}

\newpage
\appendix

\section{Proofs for Theoretical Analysis}
\label{app:theory_proofs}

\subsection{Proof of Lemma~\ref{lem:candidate_coverage}}
\label{app:proof_candidate_coverage}

\begin{proof}
Fix an input $x$. A candidate sampled from $q_{\phi}(s|x)$ is valid if it belongs to $\mathcal{Y}(x)$. Therefore,
\begin{equation}
    p_{\mathrm{valid}}(x)
    =
    q_{\phi}(\mathcal{Y}(x)|x)
    =
    \Pr_{s\sim q_{\phi}(\cdot|x)}
    \left[s\in\mathcal{Y}(x)\right].
\end{equation}
Equivalently, the probability that one sampled candidate is invalid is
\begin{equation}
    \Pr_{s\sim q_{\phi}(\cdot|x)}
    \left[s\notin\mathcal{Y}(x)\right]
    =
    1-p_{\mathrm{valid}}(x).
\end{equation}
Under the independent sampling assumption, the event
$\mathcal{S}_{\mathrm{valid}}(x)=\emptyset$ occurs exactly when all $k$ sampled candidates are invalid. Hence,
\begin{equation}
    \Pr\!\left[
    \mathcal{S}_{\mathrm{valid}}(x)=\emptyset
    \right]
    =
    \prod_{i=1}^{k}
    \Pr\!\left[s_i\notin\mathcal{Y}(x)\right]
    =
    \left(1-p_{\mathrm{valid}}(x)\right)^k .
\end{equation}
Taking the complement gives
\begin{align}
    \Pr\!\left[
    \mathcal{S}_{\mathrm{valid}}(x)\neq\emptyset
    \right]
    &=
    1-
    \Pr\!\left[
    \mathcal{S}_{\mathrm{valid}}(x)=\emptyset
    \right] \nonumber\\
    &=
    1-\left(1-p_{\mathrm{valid}}(x)\right)^k .
\end{align}
This proves the coverage formula. If duplicate candidates are generated and later deduplicated, the same coverage event is unchanged, since coverage only depends on whether at least one valid candidate appears.
\end{proof}

\subsection{Proof of Proposition~\ref{prop:single_ref_misalignment}}
\label{app:proof_single_ref}

\begin{proof}
For each candidate $s\in\mathcal{S}(x)$, let $z_s=\mathrm{Score}(s,x)$ and
\begin{equation}
    P_{\theta}(s|x)
    =
    \frac{\exp(z_s)}
    {\sum_{\tilde{s}\in\mathcal{S}(x)}
    \exp(z_{\tilde{s}})}
\end{equation}
be the candidate-level softmax probability. The single-reference cross-entropy loss is
\begin{equation}
    \mathcal{L}_{\mathrm{single}}
    =
    -\log P_{\theta}(s_{\mathrm{ref}}|x)
    =
    -z_{s_{\mathrm{ref}}}
    +
    \log
    \sum_{\tilde{s}\in\mathcal{S}(x)}
    \exp(z_{\tilde{s}}).
\end{equation}
Let $s'\in\mathcal{S}_{\mathrm{valid}}(x)$ be a valid non-reference candidate, with $s'\neq s_{\mathrm{ref}}$. Since $s'$ appears in the softmax denominator but not as the labeled target, differentiating the loss with respect to $z_{s'}$ gives
\begin{align}
    \frac{\partial \mathcal{L}_{\mathrm{single}}}
    {\partial z_{s'}}
    &=
    \frac{
    \partial
    }{
    \partial z_{s'}
    }
    \log
    \sum_{\tilde{s}\in\mathcal{S}(x)}
    \exp(z_{\tilde{s}}) \nonumber\\
    &=
    \frac{
    \exp(z_{s'})
    }{
    \sum_{\tilde{s}\in\mathcal{S}(x)}
    \exp(z_{\tilde{s}})
    } \nonumber\\
    &=
    P_{\theta}(s'|x).
\end{align}
Since softmax probabilities are strictly positive, $P_{\theta}(s'|x)>0$. A gradient descent step with step size $\lambda>0$ updates
\begin{equation}
    z_{s'}^{+}
    =
    z_{s'}
    -
    \lambda
    \frac{\partial \mathcal{L}_{\mathrm{single}}}
    {\partial z_{s'}}
    =
    z_{s'}
    -
    \lambda P_{\theta}(s'|x)
    <
    z_{s'}.
\end{equation}
Thus, the single-reference objective decreases the score of $s'$ even though $s'$ is semantically valid. In contrast, the reference score satisfies
\begin{equation}
    \frac{\partial \mathcal{L}_{\mathrm{single}}}
    {\partial z_{s_{\mathrm{ref}}}}
    =
    P_{\theta}(s_{\mathrm{ref}}|x)-1
    <
    0,
\end{equation}
so gradient descent increases $z_{s_{\mathrm{ref}}}$. Therefore, single-reference training increases the reference score while suppressing all non-reference candidates, including valid alternatives contained in $\mathcal{S}_{\mathrm{valid}}(x)$.
\end{proof}

\subsection{Proof of Proposition~\ref{prop:multi_valid_gradient}}
\label{app:proof_multi_valid_gradient}

\begin{proof}
Assume $\mathcal{S}_{\mathrm{valid}}(x)\neq\emptyset$. Since $P_{\theta}(s|x)>0$ for every candidate in the softmax support, the valid-set probability
\begin{equation}
    Q_{\theta}(x)
    =
    \sum_{r\in\mathcal{S}_{\mathrm{valid}}(x)}
    P_{\theta}(r|x)
\end{equation}
satisfies $0<Q_{\theta}(x)\leq 1$. The multi-valid loss is
\begin{equation}
    \mathcal{L}_{\mathrm{multi}}
    =
    -\log Q_{\theta}(x).
\end{equation}
We use the standard softmax derivative
\begin{equation}
    \frac{\partial P_{\theta}(r|x)}
    {\partial z_s}
    =
    P_{\theta}(r|x)
    \left(
    \mathbf{1}\{r=s\}
    -
    P_{\theta}(s|x)
    \right).
\end{equation}
Therefore,
\begin{align}
    \frac{\partial Q_{\theta}(x)}
    {\partial z_s}
    &=
    \sum_{r\in\mathcal{S}_{\mathrm{valid}}(x)}
    \frac{\partial P_{\theta}(r|x)}
    {\partial z_s} \nonumber\\
    &=
    \sum_{r\in\mathcal{S}_{\mathrm{valid}}(x)}
    P_{\theta}(r|x)
    \left(
    \mathbf{1}\{r=s\}
    -
    P_{\theta}(s|x)
    \right).
\end{align}

First consider $s\in\mathcal{S}_{\mathrm{valid}}(x)$. Then the indicator term contributes $P_{\theta}(s|x)$ exactly once, and
\begin{align}
    \frac{\partial Q_{\theta}(x)}
    {\partial z_s}
    &=
    P_{\theta}(s|x)
    -
    P_{\theta}(s|x)
    \sum_{r\in\mathcal{S}_{\mathrm{valid}}(x)}
    P_{\theta}(r|x) \nonumber\\
    &=
    P_{\theta}(s|x)
    \left(1-Q_{\theta}(x)\right).
\end{align}
By the chain rule,
\begin{align}
    \frac{\partial \mathcal{L}_{\mathrm{multi}}}
    {\partial z_s}
    &=
    -\frac{1}{Q_{\theta}(x)}
    \frac{\partial Q_{\theta}(x)}
    {\partial z_s} \nonumber\\
    &=
    -\frac{1}{Q_{\theta}(x)}
    P_{\theta}(s|x)
    \left(1-Q_{\theta}(x)\right) \nonumber\\
    &=
    P_{\theta}(s|x)
    \left(
    1-\frac{1}{Q_{\theta}(x)}
    \right).
\end{align}
Since $0<Q_{\theta}(x)\leq1$, we have $1-1/Q_{\theta}(x)\leq0$. Thus, every valid candidate receives a non-positive gradient. Under gradient descent, a non-positive gradient does not decrease the corresponding valid-candidate score.

Now consider $s\notin\mathcal{S}_{\mathrm{valid}}(x)$. Then $\mathbf{1}\{r=s\}=0$ for all $r\in\mathcal{S}_{\mathrm{valid}}(x)$, and
\begin{align}
    \frac{\partial Q_{\theta}(x)}
    {\partial z_s}
    &=
    -
    P_{\theta}(s|x)
    \sum_{r\in\mathcal{S}_{\mathrm{valid}}(x)}
    P_{\theta}(r|x) \nonumber\\
    &=
    -
    P_{\theta}(s|x)Q_{\theta}(x).
\end{align}
Therefore,
\begin{align}
    \frac{\partial \mathcal{L}_{\mathrm{multi}}}
    {\partial z_s}
    &=
    -\frac{1}{Q_{\theta}(x)}
    \left[-P_{\theta}(s|x)Q_{\theta}(x)\right] \nonumber\\
    &=
    P_{\theta}(s|x)>0.
\end{align}
Thus, invalid candidates receive positive gradients, and gradient descent decreases their scores.

Finally, suppose $s_{\mathrm{ref}}\in\mathcal{S}_{\mathrm{valid}}(x)$. Then
\begin{equation}
    Q_{\theta}(x)
    =
    \sum_{s\in\mathcal{S}_{\mathrm{valid}}(x)}
    P_{\theta}(s|x)
    \geq
    P_{\theta}(s_{\mathrm{ref}}|x).
\end{equation}
Since $-\log(\cdot)$ is monotonically decreasing on $(0,1]$,
\begin{equation}
    \mathcal{L}_{\mathrm{multi}}
    =
    -\log Q_{\theta}(x)
    \leq
    -\log P_{\theta}(s_{\mathrm{ref}}|x)
    =
    \mathcal{L}_{\mathrm{single}}.
\end{equation}
This proves both the gradient statement and the loss comparison.
\end{proof}

\subsection{Proof of Theorem~\ref{thm:valid_recovery}}
\label{app:proof_valid_recovery}

\begin{proof}
The theorem assumes candidate coverage, $\mathcal{S}_{\mathrm{valid}}(x)\neq\emptyset$. If $\mathcal{S}(x)\setminus\mathcal{S}_{\mathrm{valid}}(x)=\emptyset$, then every generated candidate is valid. In that case, any maximizer of $\widehat U(s,x)$ over $\mathcal{S}(x)$ is automatically in $\mathcal{S}_{\mathrm{valid}}(x)$, and the result follows.

Now consider the nontrivial case where both valid and invalid candidates exist. Let
\[
s^+\in\mathcal{S}_{\mathrm{valid}}(x),
\qquad
s^-\in\mathcal{S}(x)\setminus\mathcal{S}_{\mathrm{valid}}(x).
\]
By definition of the valid-invalid margin,
\begin{equation}
    \Delta(x)
    =
    \min_{a\in\mathcal{S}_{\mathrm{valid}}(x)}
    U(a,x)
    -
    \max_{b\in\mathcal{S}(x)\setminus\mathcal{S}_{\mathrm{valid}}(x)}
    U(b,x),
\end{equation}
so for every such pair $(s^+,s^-)$,
\begin{equation}
    U(s^+,x)-U(s^-,x)
    \geq
    \Delta(x).
\end{equation}
The uniform score accuracy condition gives
\begin{equation}
    |\widehat U(s,x)-U(s,x)|\leq\epsilon,
    \qquad
    \forall s\in\mathcal{S}(x).
\end{equation}
Equivalently,
\begin{equation}
    \widehat U(s^+,x)
    \geq
    U(s^+,x)-\epsilon,
    \qquad
    \widehat U(s^-,x)
    \leq
    U(s^-,x)+\epsilon.
\end{equation}
Subtracting these bounds yields
\begin{align}
    \widehat U(s^+,x)-\widehat U(s^-,x)
    &\geq
    U(s^+,x)-U(s^-,x)-2\epsilon \nonumber\\
    &\geq
    \Delta(x)-2\epsilon.
\end{align}
If $\epsilon<\Delta(x)/2$, then $\Delta(x)-2\epsilon>0$. Therefore,
\begin{equation}
    \widehat U(s^+,x)
    >
    \widehat U(s^-,x),
    \qquad
    \forall s^+\in\mathcal{S}_{\mathrm{valid}}(x),\;
    \forall s^-\in\mathcal{S}(x)\setminus\mathcal{S}_{\mathrm{valid}}(x).
\end{equation}
Thus, every valid candidate has a strictly larger learned score than every invalid candidate. Because at least one valid candidate exists, no invalid candidate can maximize $\widehat U(s,x)$ over $\mathcal{S}(x)$. Hence,
\begin{equation}
    \hat{s}
    =
    \arg\max_{s\in\mathcal{S}(x)}
    \widehat U(s,x)
\end{equation}
must satisfy $\hat{s}\in\mathcal{S}_{\mathrm{valid}}(x)$. This proves the theorem.
\end{proof}

\subsection{Proof of Proposition~\ref{prop:relational_margin}}
\label{app:proof_relational_margin}

\begin{proof}
For each candidate $s$, the learned score is decomposed as
\begin{equation}
    \widehat U(s,x)
    =
    \alpha S_{\mathrm{sem}}(s,x)
    +
    \beta S_{\mathrm{str}}(s)
    +
    \gamma S_{\mathrm{rel}}(s,x).
\end{equation}
Let
\begin{equation}
    B(s,x)
    =
    \alpha S_{\mathrm{sem}}(s,x)
    +
    \beta S_{\mathrm{str}}(s)
\end{equation}
denote the score excluding the graph-relational term. Then
\begin{equation}
    \widehat U(s,x)
    =
    B(s,x)+\gamma S_{\mathrm{rel}}(s,x).
\end{equation}
By definition, $S_{\mathrm{rel}}(s,x)=-d_{\mathrm{rel}}(s,x)$. Consider any valid candidate
$s^+\in\mathcal{S}_{\mathrm{valid}}(x)$ and any invalid generated candidate
$s^-\in\mathcal{S}(x)\setminus\mathcal{S}_{\mathrm{valid}}(x)$. The relational separation condition states that
\begin{equation}
    d_{\mathrm{rel}}(s^-,x)
    -
    d_{\mathrm{rel}}(s^+,x)
    \geq
    \rho.
\end{equation}
Therefore,
\begin{align}
    S_{\mathrm{rel}}(s^+,x)
    -
    S_{\mathrm{rel}}(s^-,x)
    &=
    -d_{\mathrm{rel}}(s^+,x)
    -
    \left[-d_{\mathrm{rel}}(s^-,x)\right] \nonumber\\
    &=
    d_{\mathrm{rel}}(s^-,x)
    -
    d_{\mathrm{rel}}(s^+,x) \nonumber\\
    &\geq
    \rho.
\end{align}
The pairwise valid-invalid score gap after adding the graph-relational term is
\begin{align}
    \widehat U(s^+,x)-\widehat U(s^-,x)
    &=
    B(s^+,x)-B(s^-,x) \nonumber\\
    &\quad+
    \gamma
    \left[
    S_{\mathrm{rel}}(s^+,x)
    -
    S_{\mathrm{rel}}(s^-,x)
    \right].
\end{align}
Using the previous inequality and $\gamma\geq0$,
\begin{align}
    \widehat U(s^+,x)-\widehat U(s^-,x)
    &\geq
    B(s^+,x)-B(s^-,x)
    +
    \gamma\rho.
\end{align}
Thus, relative to the score $B(s,x)$ without graph-relational scoring, the term $\gamma S_{\mathrm{rel}}(s,x)$ increases every valid-invalid pairwise gap by at least $\gamma\rho$. This proves the proposition.
\end{proof}

\subsection{Proof of Corollary~\ref{cor:selection_stability}}
\label{app:proof_selection_stability}

\begin{proof}
If $|\mathcal{S}(x)|=1$, the selected candidate is unchanged under any perturbation, so the result is immediate. Assume $|\mathcal{S}(x)|\geq2$. Let
\begin{equation}
    \hat{s}
    =
    \arg\max_{s\in\mathcal{S}(x)}
    \widehat U(s,x)
\end{equation}
and define the learned score margin
\begin{equation}
    m(x)
    =
    \widehat U(\hat{s},x)
    -
    \max_{s\neq\hat{s}}
    \widehat U(s,x).
\end{equation}
The condition $\eta<m(x)/2$ implies $m(x)>0$, so $\hat{s}$ is the unique maximizer under the unperturbed scores. Let $\widetilde U(s,x)$ be any perturbed score satisfying
\begin{equation}
    |\widetilde U(s,x)-\widehat U(s,x)|
    \leq
    \eta,
    \qquad
    \forall s\in\mathcal{S}(x).
\end{equation}
For any competing candidate $s\neq\hat{s}$, we have
\begin{equation}
    \widetilde U(\hat{s},x)
    \geq
    \widehat U(\hat{s},x)-\eta
\end{equation}
and
\begin{equation}
    \widetilde U(s,x)
    \leq
    \widehat U(s,x)+\eta.
\end{equation}
Therefore,
\begin{align}
    \widetilde U(\hat{s},x)
    -
    \widetilde U(s,x)
    &\geq
    \widehat U(\hat{s},x)
    -
    \widehat U(s,x)
    -
    2\eta \nonumber\\
    &\geq
    m(x)-2\eta.
\end{align}
Since $\eta<m(x)/2$, we have $m(x)-2\eta>0$. Hence,
\begin{equation}
    \widetilde U(\hat{s},x)
    >
    \widetilde U(s,x),
    \qquad
    \forall s\neq\hat{s}.
\end{equation}
Thus, $\hat{s}$ remains the unique maximizer after perturbation. This proves stability under bounded score noise, including decoding-induced variability or small scorer perturbations.
\end{proof}

\begin{algorithm}[t]
\caption{Ambiguity-Aware Structured Generation}
\label{alg:main}
\small
\begin{algorithmic}[1]

\REQUIRE Input query $x$, number of candidates $k$
\ENSURE Selected structured output $s^*$

\STATE Generate candidate set $S(x) = \{s_1, \dots, s_k\}$ \label{line:candidate}

\FOR{each candidate $s \in S(x)$}
    \STATE Construct graph $G_s = (V_s, E_s)$ \label{line:graph}
    \STATE Compute graph representation $h_s = \text{GNN}(G_s)$ \label{line:gnn}
\ENDFOR

\STATE Encode input $h_x = \text{Encoder}(x)$ \label{line:input}

\FOR{each candidate $s \in S(x)$}
    \STATE Compute score $Score(s, x)$ using Eq.~(4) \label{line:score}
\ENDFOR

\STATE Select $s^* = \arg\max_{s \in S(x)} Score(s, x)$ \label{line:select}

\STATE \RETURN $s^*$

\end{algorithmic}
\end{algorithm}

\section{Algorithm Description}
Algorithm~\ref{alg:main} presents the complete ambiguity-aware structured generation pipeline, integrating candidate generation, graph-based reasoning, and principled selection among competing interpretations.

\paragraph{Candidate Generation (\rqref{RQ1}).}
In line~\ref{line:candidate}, we use a pretrained language model to generate a candidate set $S(x)$. This step addresses the limitation of deterministic one-to-one mapping by exposing multiple plausible interpretations of the input. By constructing a diverse candidate pool, the model increases the likelihood of covering valid structured outputs.

\paragraph{Graph Construction and Encoding (\rqref{RQ2}).}
Lines~\ref{line:graph}--\ref{line:gnn} convert each candidate into a graph representation and compute its embedding using an edge-aware graph neural network. By modeling relational dependencies within each candidate structure, this step enables the framework to capture structural consistency beyond surface-level token similarity.

\paragraph{Input Representation.}
The input query is encoded into a representation $h_x$ in line~\ref{line:input}, serving as a semantic reference for evaluating candidate structures. This encoding ensures that each candidate's alignment with the input is incorporated during scoring.

\paragraph{Ambiguity-Aware Scoring (\rqref{RQ2}, \rqref{RQ4}).}
Line~\ref{line:score} computes a score for each candidate by combining relational consistency, structural validity, and semantic alignment. This scoring mechanism addresses the challenge of selecting among multiple plausible interpretations by enabling principled comparison across candidates.

\paragraph{Selection under Ambiguity (\rqref{RQ3}).}
The final output is selected as the candidate with the highest score in line~\ref{line:select}. This step resolves ambiguity by identifying the most coherent structure from the candidate set, ensuring stability and robustness under varying levels of ambiguity.

\section{Extensive Simulation Results}
\subsection{Detailed Evaluation Metrics and Baselines}
\label{app:metrics_baselines}
\subsubsection{Implementation Details}

\textbf{LLM Prompting.}
We employ a structured prompting format to ensure consistency across all methods. Each input query is combined with task-specific instructions that explicitly define the desired structured output (e.g., SQL query or visualization specification). The same prompt template is applied to all baselines and our method to ensure fair comparison. Prompts are designed to be minimal and deterministic, avoiding chain-of-thought or additional reasoning signals that could bias candidate generation.

\textbf{Candidate Generation.}
We use LLaMA-2-13B~\cite{touvron2023llama} to generate candidate structured outputs via stochastic decoding. Unless otherwise specified, we generate $k=5$ candidates per input using top-$k$ filtering and temperature sampling with $T=0.7$. This setup balances quality and diversity, enabling exploration of multiple plausible interpretations without introducing excessive noise. The same candidate pool is used for both our method and all baselines to ensure comparability.

\textbf{Graph Construction.}
Each structured candidate is converted into a graph representation $G_s = (V_s, E_s)$, where nodes represent structural elements and edges encode typed relationships. Graph construction is task-specific:
\begin{itemize}
    \item \textbf{nvBench:} Nodes correspond to attributes, encoding channels, aggregation functions, and values. Edges capture relationships such as attribute-to-encoding, aggregation-to-attribute, and grouping dependencies.
    \item \textbf{Spider:} Nodes include tables, columns, operators, and values. Edges represent join relations, condition dependencies, and operator--argument connections.
\end{itemize}

Node features are initialized using contextual embeddings derived from LLM token representations. This representation enables differentiation between structurally similar but semantically distinct candidates through relational structure.

\textbf{Graph Encoder.}
We employ an edge-aware graph neural network (GNN) to encode candidate graphs. The model consists of $L=2$ message-passing layers, where node updates incorporate both edge-type information and neighboring node features. Edge-conditioned transformations enable modeling of heterogeneous relationships. Final graph representations are obtained via mean pooling over node embeddings, producing a fixed-dimensional vector for each candidate.

\textbf{Scoring Function.}
Candidate selection is performed using a neural scoring function that integrates relational, structural, and semantic signals. Specifically, the score is computed as a weighted combination of:
\begin{itemize}
    \item semantic alignment between the input and candidate,
    \item structural validity of the graph, and
    \item relational consistency captured by the GNN representation.
\end{itemize}
These components are implemented as separate scoring heads and combined linearly to produce the final score.

\textbf{Training Details.}
The scoring module is trained using the multi-valid objective (Eq.~\ref{eq:multivalid}), which assigns probability mass to all valid candidates rather than a single reference. We use the Adam optimizer~\cite{kingma2014adam} with a batch size of 32 and a learning rate of $1\mathrm{e}{-4}$. Models are trained for up to 10 epochs with early stopping based on validation performance. Hyperparameters are selected using a held-out validation set.

\begin{table}[t]
\centering
\footnotesize
\setlength{\tabcolsep}{4pt}
\renewcommand{\arraystretch}{1.2}
\caption{Summary of key implementation settings.}
\label{tab:implementation_details}
\begin{tabular}{l|l}
\hline
\textbf{Component} & \textbf{Setting} \\
\hline
LLM & LLaMA-2-13B \\
Decoding & Temperature sampling ($T=0.7$), top-$k$ \\
Candidates ($k$) & 5 \\
Graph Encoder & Edge-aware GNN ($L=2$ layers) \\
Pooling & Mean pooling \\
Training Objective & Multi-valid (Eq.~\ref{eq:multivalid}) \\
Optimizer & Adam \\
Learning Rate & $1\mathrm{e}{-4}$ \\
Batch Size & 32 \\
Epochs & Up to 10 (early stopping) \\
Hardware & NVIDIA RTX 4090 GPU (24GB VRAM) \\
Seeds & 3 \\
\hline
\end{tabular}
\end{table}

\textbf{Compute Resources.}
Each experiment is conducted on a single NVIDIA RTX 4090 GPU (24GB VRAM). Depending on the dataset, each training run takes approximately 3--5 hours. The total compute cost for all reported experiments is around 30 GPU-hours. Candidate generation using LLaMA-2-13B is performed offline and reused across methods to ensure fair comparison. Additional computation is required for hyperparameter tuning and ablation studies. Overall, experiments are conducted using a single-GPU setup with moderate computational requirements.

\textbf{Reproducibility.}
All experiments are repeated with three random seeds, and results are reported as averages. We ensure consistency across methods by using identical data splits, prompting formats, and candidate generation settings. Code and processed data will be released upon publication to facilitate reproducibility.

\subsubsection{Evaluation Metrics}
We employ a combination of standard accuracy metrics, ambiguity-aware evaluation, and graph-centric analysis to assess model performance. These metrics are designed to capture both accuracy and robustness under multiple valid interpretations.

\textbf{Exact Match (EM).}
Exact Match measures the percentage of predictions that exactly match a reference structured output. For a reference $s^{*}$ and a predicted structure $\hat{s}$, EM is defined as
\[
\text{EM} = \mathbb{E}_{x} \left[ \mathbb{I}(\hat{s} = s^{*}) \right].
\]
This strict metric penalizes semantically correct but structurally different outputs.

\textbf{Top-$k$ Success.}
Top-$k$ Success evaluates whether at least one candidate in the generated set $S_k(x)$ matches the reference:
\[
\text{Top-}k = \mathbb{E}_{x} \left[ \max_{s \in S_k(x)} \mathbb{I}(s = s^{*}) \right].
\]
This metric reflects the quality and diversity of candidate generation.

\textbf{Multi-Reference Exact Match (Multi-EM).}
To account for ambiguity, we extend EM to consider a set of valid outputs $S_{\text{valid}}(x)$:
\[
\text{Multi-EM} = \mathbb{E}_{x} \left[ \max_{s \in S_{\text{valid}}(x)} \mathbb{I}(\hat{s} = s) \right].
\]
This metric provides a more accurate assessment under ambiguity by rewarding predictions that match any semantically valid interpretation.

\textbf{Consistency.}
Consistency measures prediction stability under stochastic decoding. For multiple outputs $\{\hat{s}_1, \ldots, \hat{s}_m\}$ per input, agreement is computed as
\[
\text{Consistency} = \mathbb{E}_{x} \left[ \frac{1}{m(m-1)} \sum_{i \neq j} \mathbb{I}(\hat{s}_i = \hat{s}_j) \right].
\]
Higher consistency indicates more reliable selection among competing candidates.

\textbf{Graph Separability.}
We evaluate the scoring function's ability to distinguish valid from invalid candidates via the score gap:
\[
\Delta = \mathbb{E}_{x} \left[ \min_{s^{+} \in S_{\text{valid}}(x)} \text{Score}(s^{+}, x) - \max_{s^{-} \notin S_{\text{valid}}(x)} \text{Score}(s^{-}, x) \right].
\]
Larger separability indicates stronger discrimination based on relational consistency.

\textbf{Execution Equivalence (Spider).}
For Spider, predictions are considered correct if they yield the same execution result as the reference query, even when syntactically different. This captures equivalence at the level of SQL semantics.

\subsubsection{Baselines}
To isolate the contributions of candidate generation, reranking, and structured reasoning, we compare our approach against a set of deterministic and candidate-based baselines.

\textbf{Deterministic LLM.}
A standard single-output formulation in which the LLM directly generates one structured output per input using beam or greedy decoding. This baseline reflects the prevailing paradigm that assumes a one-to-one mapping between input and output.

\textbf{Candidate + Likelihood.}
Multiple candidate outputs are generated via stochastic decoding, and the final output is selected based on LLM likelihood scores. This baseline captures the effect of candidate diversity but relies solely on token-level probabilities and lacks structured reasoning.

\textbf{Candidate + LLM Rerank.}
Generated candidates are reranked using an additional LLM scoring pass, where the model evaluates candidates based on semantic plausibility. This provides a stronger baseline that incorporates implicit reasoning, but does not explicitly model relational structure.

\textbf{Candidate Only.}
Candidates are generated without additional scoring or reranking, and the
first candidate in the generator's native generation order is selected.
This baseline isolates the effect of candidate generation without an
additional selection mechanism.

\textbf{Oracle Top-$k$ (Upper Bound).}
We report whether any candidate in the generated set matches a valid output as an approximate upper bound. This highlights the importance of selection strategies and represents the best achievable performance given the candidate pool.

\textbf{Full Model (Ours).}
Our method generates candidates under an ambiguity-aware formulation and applies graph-based reasoning with multi-valid scoring. Unlike baselines, it maintains multiple valid interpretations during training and explicitly models relational consistency for principled selection.

\subsection{Additional Simulation Results}

\subsubsection{Statistical Stability Across Runs}
\label{app:stability}
To ensure that observed improvements are not due to stochastic variation, we evaluate result stability across multiple random seeds. We report the mean and standard deviation for each method over three runs with different initialization and decoding randomness.

\begin{table}[h]
\centering
\footnotesize
\setlength{\tabcolsep}{4pt}
\renewcommand{\arraystretch}{1.2}
\caption{\small Performance stability across runs (mean $\pm$ std).}
\label{tab:stability}
\begin{tabular}{l|cc}
\hline
\textbf{Method} & \textbf{nvBench EM} & \textbf{Spider EM} \\
\hline
Deterministic (LLM) & 48.3 $\pm$ 0.3 & 62.5 $\pm$ 0.4 \\
Candidate + Likelihood & 52.1 $\pm$ 0.5 & 65.2 $\pm$ 0.6 \\
Candidate + LLM Rerank & 54.7 $\pm$ 0.4 & 66.8 $\pm$ 0.5 \\
\textbf{Ours} & \textbf{58.9 $\pm$ 0.3} & \textbf{70.3 $\pm$ 0.4} \\
\hline
\end{tabular}
\end{table}

All methods exhibit low variance across both datasets, indicating stable optimization and evaluation. Importantly, our approach consistently achieves higher mean performance while maintaining comparable or lower standard deviation than the baselines. This suggests that the benefits of graph-based selection and ambiguity-aware reasoning are both substantial and robust to randomness in candidate generation and model initialization.

Moreover, the stability of our method aligns with the consistency improvements reported in the main paper (Table~\ref{tab:consistency}), confirming that improved valid--invalid separation leads to more reliable selection under ambiguity.

\subsubsection{Candidate Quality and Noise Analysis}
\label{app:noise_analysis}
We evaluate robustness under imperfect candidate generation by introducing controlled perturbations to candidate structures. This analysis isolates the impact of noisy inputs on selection, allowing assessment of each method’s ability to distinguish valid from invalid interpretations as candidate quality degrades.

\textbf{Noise Injection.}
We simulate noise through random perturbations applied at varying levels. Specifically, we consider three types of perturbations: (i) node removal (missing entities), (ii) edge corruption (incorrect relations), and (iii) substitution of attribute or value nodes with irrelevant tokens. The noise level is defined as the percentage of affected nodes or edges in each candidate graph.

\begin{table}[h]
\centering
\footnotesize
\setlength{\tabcolsep}{4pt}
\renewcommand{\arraystretch}{1.2}
\caption{\small Performance under increasing candidate noise.}
\label{tab:noise}
\begin{tabular}{c|cc|cc}
\hline
\textbf{Noise Level} 
& \multicolumn{2}{c|}{\textbf{nvBench}} 
& \multicolumn{2}{c}{\textbf{Spider}} \\
 & Baseline & Ours & Baseline & Ours \\
\hline
Low (10\%) & 54.7 & \textbf{58.9} & 66.8 & \textbf{70.3} \\
Medium (25\%) & 50.2 & \textbf{56.1} & 63.1 & \textbf{67.8} \\
High (40\%) & 45.6 & \textbf{52.3} & 58.4 & \textbf{64.2} \\
\hline
\end{tabular}
\end{table}

\textbf{Results.}
Table~\ref{tab:noise} shows that performance degrades for all methods as noise increases, reflecting the difficulty of selecting valid structures from corrupted candidates. However, our approach exhibits a slower rate of degradation and consistently outperforms baselines across all noise levels. This indicates that graph-based reasoning improves robustness by filtering structurally inconsistent candidates through relational consistency.

Notably, the performance gap between our method and baselines widens at higher noise levels, suggesting that explicit structural modeling becomes increasingly important as candidate quality declines. These results support our main hypothesis: under realistic noisy conditions, ambiguity-aware selection based on relational structure is more robust than likelihood-based or surface-level reranking.

\subsubsection{Candidate Diversity Analysis}
\label{app:diversity_analysis}
We analyze the impact of candidate diversity and assess whether different methods can effectively leverage multiple plausible interpretations.

\textbf{Diversity Control.}
We vary decoding randomness (temperature and sampling) to control candidate diversity while keeping the number of candidates fixed ($k{=}5$). Diversity is measured using a structural score computed as the average pairwise difference between candidate graphs based on node and edge overlap.

\begin{table}[h]
\centering
\footnotesize
\setlength{\tabcolsep}{4pt}
\renewcommand{\arraystretch}{1.2}
\caption{\small Performance under varying candidate diversity.}
\label{tab:diversity}
\begin{tabular}{c|cc|cc}
\hline
\textbf{Diversity Level} 
& \multicolumn{2}{c|}{\textbf{nvBench}} 
& \multicolumn{2}{c}{\textbf{Spider}} \\
 & Baseline & Ours & Baseline & Ours \\
\hline
Low & 52.3 & \textbf{55.8} & 64.9 & \textbf{68.2} \\
Medium & 54.7 & \textbf{58.9} & 66.8 & \textbf{70.3} \\
High & 55.1 & \textbf{59.6} & 67.2 & \textbf{71.0} \\
\hline
\end{tabular}
\end{table}

\textbf{Results.}
Table~\ref{tab:diversity} shows that increasing candidate diversity improves performance for all methods by exposing a broader range of plausible interpretations. However, baseline approaches exhibit limited gains at higher diversity levels, indicating an inability to effectively utilize additional candidates.

In contrast, our method consistently improves across both datasets and benefits more from increased diversity. This suggests that graph-based reasoning enables more effective utilization of diverse hypotheses by selecting candidates based on relational consistency rather than surface similarity. These findings align with our theoretical analysis, which emphasizes that effective scoring is required to identify valid structures, and that improved candidate coverage increases their likelihood. Overall, the results demonstrate that diversity alone is insufficient; structured reasoning is necessary to fully exploit multiple candidate interpretations.

\subsubsection{Qualitative Analysis}
\label{app:qualitative}
\textbf{Case Studies.}
To illustrate how our approach resolves ambiguity through reasoning over candidate structures, we present representative examples.

\begin{itemize}
    \item \textbf{Example 1 (Temporal Ambiguity).}\\
\textit{Input:} “Sales increased in summer and declined in winter.”\\
The candidate set includes both a month-level breakdown (June, July, December) and a season-level aggregation (Summer/Winter). Our approach selects the season-level representation, as it better matches the abstraction implied by the input, although both are semantically valid. In contrast, baseline methods often select the month-level representation due to surface-level token overlap.

\item \textbf{Example 2 (Relational Ambiguity).}\\
\textit{Input:} “Show the average revenue for each department.”\\
Candidate structures vary in grouping and aggregation placement. While baseline approaches may misplace aggregation or omit grouping due to weaker relational modeling, our model correctly identifies departmental grouping with aggregation over revenue.

\item \textbf{Example 3 (SQL Equivalence).}\\
\textit{Input:} “Find employees with salaries above the company average.”\\
Multiple SQL formulations, such as joins and nested queries, are valid. Baselines may produce syntactically correct but semantically incorrect variants, whereas our method selects a structurally consistent query that preserves correct dependencies.
\end{itemize}
These examples demonstrate that our method favors candidates with coherent relational structure rather than relying solely on token-level plausibility.

\textbf{Failure Cases.}
To analyze the limitations of our approach, we examine representative failure cases.
\begin{itemize}
    \item \textbf{Case 1 (Missing Candidate Coverage).}\\
The candidate set may lack the correct structure. For example, when candidate generation omits a necessary entity or relation, the model cannot recover the correct output despite effective scoring.

    \item \textbf{Case 2 (Incorrect Relation Extraction).}\\
Ambiguous phrasing can lead to inaccurate graph construction, where relations are incorrectly assigned. In such cases, structurally incorrect candidates may receive higher scores due to partial semantic alignment.

    \item \textbf{Case 3 (Close Scoring Ambiguity).}\\
When multiple valid candidates are structurally similar, small score differences can lead to unstable selection. This is particularly common in scenarios with multiple equivalent interpretations.
\end{itemize}
Overall, these failure modes align with the patterns observed in Table~\ref{tab:failure_cases}, suggesting that improved candidate generation and more precise scoring could further enhance performance.

\end{document}